\RequirePackage{fix-cm}
\documentclass[11pt,letterpaper,logo,onecolumn,copyright]{microsoft}

\usepackage{nicefrac}
\usepackage{float}
\usepackage{natbib}
\hypersetup{
  pdftitle={Holdout Best-of-N: Unbiased Evaluation and Its Cost},
  pdfauthor={Shrey Shah and Yinheng Li}
}

\newtheorem{theorem}{Theorem}

\newtheorem{proposition}{Proposition}
\newtheorem{corollary}{Corollary}
\newtheorem{definition}{Definition}
\newtheorem{assumption}{Assumption}

\newcommand{\E}{\mathbb{E}}
\newcommand{\Var}{\operatorname{Var}}

\newcommand{\hbon}{\widehat{\Theta}_{\mathrm{H}}}

\title{Holdout Best-of-$N$: Unbiased Evaluation and Its Cost}

\author{Shrey Shah and Yinheng Li}

\begin{document}

\begin{abstract}
Reusing the scores that select a Best-of-$N$ winner can overstate its expected reward. We study evaluation from a fixed matrix of $K$ independent scores per candidate for a policy that selects using $J$ fresh scores. A single estimator based only on this matrix is exactly unbiased for expected judge reward under every independent, stable collection of candidate-specific score laws if and only if $J<K$, for every pool size $M\ge N\ge2$. At $J=K-1$, the selector deepens as $K$ grows. For independent Gaussian scores with common variance and fixed $M\ge N\ge2$, the unbiased minimax risk in this regime is of order $\sigma^2/\sqrt K$, attained by Holdout; allowing bias improves the rate to $\sigma^2/K$. For two candidates, we derive the minimum-variance unbiased estimator at known variance and the sharp asymptotic unbiased minimax constant $1/(\pi\sqrt2)$, which Holdout attains without knowing the variance. The cyclic average over subsets and ties can be computed in $O(MK\log M)$ operations. At fixed selector depth, cyclic evaluation of bounded scores has $O(K^{-1})$ risk uniformly in pool size. The impossibility result concerns the fixed matrix: one additional fresh winner score permits unbiased evaluation of the all-$K$ policy.
\end{abstract}

\maketitle

\section{Introduction}
Best-of-$N$ (BoN) returns the highest-scoring generated candidate \citep{nakano2021webgpt,stiennon2020summarize,lightman2023letsverify}. With a stochastic judge, selection favors unusually high score realizations: the observed maximum can rise with $N$ even when all candidates have the same mean. Independent evaluation avoids reusing selection noise. Judge reward can still disagree with human preferences \citep{gao2023scaling,coste2023reward,eisenstein2023helping}.

We observe $K$ scores for each of $M$ candidates and evaluate a policy that selects from a uniform size-$N$ subset using $J$ fresh scores per candidate. Its value is the selected candidate's expected judge mean. We ask whether this value admits an exactly unbiased estimate from the fixed matrix, and what unbiasedness costs in mean-squared error.

Figure~\ref{fig:method-overview} illustrates column splitting. Holdout uses $J=K-1$ columns for selection and leaves one for evaluation, averaging over held-out columns, candidate subsets, and tied winners. It targets the $(K-1)$-score policy, whose value may differ from the all-$K$ policy.

\begin{figure}[t]
\centering
\includegraphics[width=\linewidth]{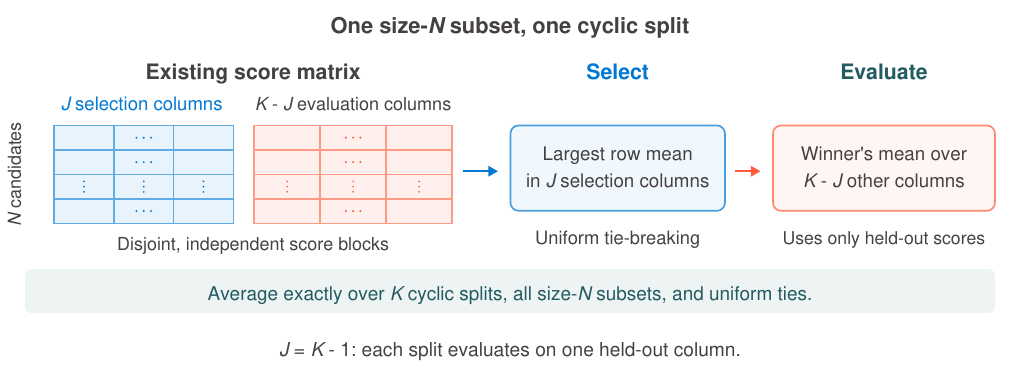}
\caption{Selection and evaluation use disjoint score columns. Selection uses the mean of $J$ columns; evaluation averages the winner's remaining $K-J$ scores. The estimator averages over cyclic splits, all size-$N$ subsets, and uniform ties. Holdout uses $J=K-1$.}
\label{fig:method-overview}
\end{figure}

Over the class of independent, stable candidate-specific score laws, unbiased evaluation from the fixed matrix is possible if and only if $J<K$. The independent-split construction follows classical unbiased functional estimation \citep{halmos1946unbiased,hoeffding1948class}; a polynomial-degree argument for the BoN functional proves necessity. Acquiring a fresh winner score changes the observation design and permits unbiased evaluation of the all-$K$ policy.

For independent Gaussian scores with common variance and fixed $M\ge N\ge2$, the unbiased minimax risk along $J=K-1$ is of order $\sigma^2/\sqrt K$. Holdout attains this rate without knowing the variance; allowing bias gives order $\sigma^2/K$. Here selector depth grows with the observed matrix. With fixed selector depth $J$ and bounded scores, cyclic evaluation has $O(K^{-1})$ risk uniformly in pool size, including discrete and unequal-variance score laws.

Section~\ref{sec:setup} defines the policy and observed data. Section~\ref{sec:method} constructs Holdout and establishes the sample boundary; Section~\ref{sec:gaussian-risk} quantifies the cost of unbiasedness. Section~\ref{sec:practical} covers evaluation budgets, confidence intervals, and dependence. Section~\ref{sec:experiments} reports the synthetic experiments.

\section{Problem setup}
\label{sec:setup}
For each prompt, condition on its realized pool of $M$ candidate responses and any latent variables that determine their score laws; denote this conditioning information by $\mathcal F$. The observed matrix $R\in\mathbb R^{M\times K}$ contains $K$ scores per candidate, and candidate $m$ has judge mean $\mu_m$.

\begin{assumption}[Independent, stable scores]
\label{ass:indep}
Conditional on $\mathcal F$, all $MK$ scores are mutually independent and integrable. For each candidate $m$, the scores $R_{m,1},\ldots,R_{m,K}$ have the same distribution and mean $\mu_m$. Different candidates may have different score distributions.
\end{assumption}

\begin{definition}[$J$-score policy value]
\label{def:achievable}
Draw a size-$N$ subset $S$ uniformly from the $M$ candidates. Within it, select $W_J(S)$ by the largest mean of $J$ fresh, independent scores per candidate, breaking ties uniformly. Define
\[
\Theta^{(J)}(M,N,\boldsymbol\mu,F_\varepsilon)
:=\E[\mu_{W_J(S)}\mid\mathcal F].
\]
Here $F_\varepsilon$ is the collection of candidate-specific centered score laws. The hypothetical draws define the policy; estimation uses only the observed score matrix.
\end{definition}

Increasing $J$ can decrease reward: with two means $0,0.5$ and iid noise uniform on $\{-1,+1\}$, the values are $\Theta^{(1)}=0.375$ and $\Theta^{(2)}=11/32$. For $N=1$, every target is the pool's mean.

\paragraph{From a fixed pool to a generated pool.}
For exchangeable candidate/score-law pairs conditional on a prompt, a uniform size-$N$ subset has the same joint law as the first $N$ candidates. Averaging $\Theta^{(J)}$ over generated pools therefore gives the population value of that selector. This includes iid generation. A fixed four-plus-four mixture of two generation policies differs from an iid single-policy pool; a production generator with a different joint law need not have the same value. Inference uses complete prompt/pool units (Section~\ref{sec:inference}).

\section{The Holdout Best-of-$N$ estimator}
\label{sec:method}
For $K=3$ and $J=2$, choose a winner from a fixed candidate subset using columns 1 and 2, then evaluate it with column 3. Repeat using columns 2 and 3 to select and column 1 to evaluate, then columns 3 and 1 to select and column 2 to evaluate. Average the three evaluations. Each winner is chosen without its evaluation column. Averaging over candidate subsets and uniform tie choices gives Holdout at three columns; the cyclic estimator below allows any $1\le J<K$.

Fix $1\le J<K$. For each of the $K$ cyclic blocks $I_t$ of $J$ columns, form a selection mean $a_{m,t}$ and an evaluation mean $v_{m,t}$ from its complement:
\[
I_t=\{1+((t+r)\bmod K):r=0,\ldots,J-1\},\qquad
a_{m,t}=\frac1J\sum_{k\in I_t}R_{m,k},\quad
v_{m,t}=\frac1{K-J}\sum_{k\notin I_t}R_{m,k}.
\]
Here $t=0,\ldots,K-1$.

\paragraph{Subset and tie averaging.}
An ascending rank $i$ wins $\binom{i-1}{N-1}$ of the $\binom MN$ subsets, so its weight is $w_i=\binom{i-1}{N-1}/\binom MN$, zero for $i<N$. Average these weights within each tied block. Writing $q_{m,t}$ for this weight gives
\begin{equation}
\label{eq:fixed-J-main}
\widehat\Theta_{J,K}=\frac1K\sum_{t=0}^{K-1}\sum_{m=1}^M q_{m,t}v_{m,t}.
\end{equation}
The weighted sum is the exact average over subsets and uniformly chosen tied winners. Rolling column sums and one sort per block compute this average in $O(MK\log M)$ arithmetic operations. The subset weights extend the binary pass@$k$ counting identity \citep{chen2021codex}; Appendix~\ref{app:computation} derives the weights and computational cost.

\paragraph{Holdout Best-of-$N$.}
At $J=K-1$, each split leaves one column for evaluation. Write $h_m^{(-k)}=(K-1)^{-1}\sum_{j\ne k}R_{m,j}$, let $\sigma^{(k)}$ sort these keys, and let $\tilde w_i^{(k)}$ average $w_i$ within ties. Equation~\ref{eq:fixed-J-main} becomes
\begin{equation}
\label{eq:holdout}
\hbon(R;N)=\frac1K\sum_{k=1}^K\sum_{i=1}^M\tilde w_i^{(k)}R_{\sigma^{(k)}(i),k}.
\end{equation}
For discrete scores, tied weights must be averaged; stable sorting alone favors the index order (Appendix~\ref{app:computation}).

\begin{theorem}[Unbiased evaluation]
\label{thm:unbiased}
Under Assumption~\ref{ass:indep}, for every $1\le J<K$,
$\E[\widehat\Theta_{J,K}\mid\mathcal F]=\Theta^{(J)}$.
In particular, $\E[\hbon\mid\mathcal F]=\Theta^{(K-1)}$.
\end{theorem}
\begin{proof}
Conditional on the selection keys $a_t$ and $\mathcal F$, each evaluation mean has conditional mean $\mu_m$ and the weights are fixed. The block's expected contribution is therefore $\sum_mq_{m,t}\mu_m$. Its keys have the joint law of the intended independent $J$-score selector, so expectation over them gives $\Theta^{(J)}$. Linearity of expectation preserves this value when averaging over the possibly dependent block estimates.
\end{proof}

\subsection{The sample boundary}
\label{sec:frontier}
\begin{theorem}[Estimability frontier]
\label{thm:frontier}
Fix $K\ge2$ and any $M\ge N\ge 2$. Over the distribution-free class of independent, stable score laws (Assumption~\ref{ass:indep}, with laws allowed to differ across candidates), consider a single statistic $T(R;M,N,J,K)$, integrable under every law in this class, with no candidate means or score-law parameters supplied as side information:
\begin{enumerate}
  \item[(a)] \textbf{Achievability.} For every integer $1\le J<K$, $\Theta^{(J)}(M,N,\cdot)$ admits an exactly unbiased estimator from the $M\times K$ sample matrix. At $J=K-1$, Holdout BoN is one (Theorem~\ref{thm:unbiased}).
  \item[(b)] \textbf{Impossibility.} For every $J\ge K$, $\Theta^{(J)}(M,N,\cdot)$ admits no exactly unbiased estimator that is a function of the $M\times K$ sample matrix.
\end{enumerate}
Thus, for integer $J\ge1$, $\Theta^{(J)}$ is distribution-free exactly unbiasedly estimable from this matrix if and only if $J<K$. The threshold specifies the selector sample count for exact unbiased estimability. Reward $\Theta^{(J)}$ can decrease with $J$.
\end{theorem}

\begin{proof}[Proof sketch]
For $M=N=2$ with Bernoulli judges, the expectation of a statistic based on $K$ samples has degree at most $K$ in each mean. The selector target has degree exactly $J+1$ in each mean, ruling out $J\ge K$. Adding $M-2$ constant, lower-scoring candidates preserves this higher-degree term with positive coefficient $\binom{M-2}{N-2}/\binom MN$. Appendix~\ref{app:frontier} proves these claims. For $J<K$, selection on $J$ columns and independent evaluation on the remaining $K-J$ columns gives an unbiased estimator; cyclic averaging preserves it (Theorem~\ref{thm:unbiased}).
\end{proof}

The same boundary holds for two candidates with known common Gaussian noise: no estimator integrable at every mean pair is unbiased for $\Theta^{(J)}$ when $J\ge K$ (Proposition~\ref{prop:gaussian-frontier}).

\section{The cost of exact unbiasedness}
At $J=K-1$, increasing $K$ also deepens the selector; at fixed $J$, it provides more data to evaluate the same policy. We first quantify Gaussian risk at the boundary, then give fixed-depth bounds for bounded scores, including discrete and heteroscedastic laws.
\label{sec:gaussian-risk}
\begin{theorem}[Gaussian risk cost for any fixed pool and subset size]
\label{thm:general-gaussian-risk}
Fix $M\ge N\ge2$. Observe independent
$R_{m,k}\sim\mathcal N(\mu_m,\sigma^2)$, with $K\ge2$ and common $\sigma>0$,
and take $J=K-1$.
Write $L_K(U;\boldsymbol\mu,\sigma)=
\E_{\boldsymbol\mu,\sigma}[(U-\Theta^{(K-1)})^2]$.
For fixed known $\sigma$, let $\mathcal U_K(\sigma)$ contain the statistics
integrable and unbiased for this target at every $\boldsymbol\mu\in\mathbb R^M$.
As $K\to\infty$,
\[
\inf_{U\in\mathcal U_K(\sigma)}\sup_{\boldsymbol\mu}L_K(U;\boldsymbol\mu,\sigma)
\asymp_{M,N}
\sup_{\boldsymbol\mu}L_K(\hbon;\boldsymbol\mu,\sigma)
\asymp_{M,N}\frac{\sigma^2}{\sqrt K}.
\]
If bias is allowed, the minimax rate is $\sigma^2/K$.
For unknown $\sigma$, the same orders, $K^{-1/2}$ and $K^{-1}$, hold after
normalizing the loss by $\sigma^2$ and taking the supremum over all
$\boldsymbol\mu\in\mathbb R^M,\sigma>0$; the unbiased class then requires
unbiasedness at every such pair without the true $\sigma$ as an input.
Holdout attains the unbiased order without a variance input or estimate.
The comparison constants may depend on $M,N$.
\end{theorem}

The lower bound uses two tied best candidates with the other means
receding below them. It gives a worst-case rate for each fixed $M,N$;
the risk when all $M$ means tie and the rate for growing pools remain separate questions.
Appendix~\ref{app:general-gaussian-risk} gives finite-$K$ bounds.

\subsection{Sharp constants for two candidates}
Assume $M=N=2$ and independent $R_{m,k}\sim\mathcal N(\mu_m,\sigma^2)$, with known $\sigma>0$. Put $S=(\bar R_1+\bar R_2)/2$, $D=\bar R_1-\bar R_2$, and $\tau^2=2\sigma^2/K$. For $d=\mu_1-\mu_2$, the target is
\[
\Theta^{(J)}=\frac{\mu_1+\mu_2}{2}+\frac d2\operatorname{erf}\!\left(\frac{\sqrt J\,d}{2\sigma}\right),
\qquad \operatorname{erf}(z)=\frac2{\sqrt\pi}\int_0^z e^{-u^2}du.
\]

\begin{theorem}[Gaussian unbiased risk]
\label{thm:gaussian-risk}
For $1\le J<K$, define $b=\sqrt{JK}/(2\sigma\sqrt{K-J})$. The estimator
\begin{equation}
\label{eq:gaussian-umvu}
T_{J,K}=S+\frac D2\operatorname{erf}(bD)
-\frac{\tau^2b}{\sqrt\pi}e^{-b^2D^2}
\end{equation}
is unbiased for every mean pair and has minimum variance at each pair among estimators unbiased throughout this Gaussian model with finite variance. At $\mu_1=\mu_2$, let $r=J/K$; its mean-squared error is
\begin{equation}
\label{eq:gaussian-tie-risk}
\mathcal R_{J,K}
=\frac{\sigma^2}{K}\left[\frac12+
\frac1\pi\left(\arcsin r+\frac r{\sqrt{1-r^2}}\right)\right].
\end{equation}
Consequently, at equal means, every estimator unbiased for all mean pairs has MSE at least $\mathcal R_{J,K}$. At $J=K-1$, this minimum is asymptotic to $\sigma^2/(\pi\sqrt{2K})$; for fixed $J$, it is asymptotic to $\sigma^2/(2K)$.
\end{theorem}

The final term in Equation~\ref{eq:gaussian-umvu} corrects the bias of the preceding nonlinear term. As $J/K$ approaches one, it concentrates near tied observed means. Appendix~\ref{app:gaussian-risk} proves unbiasedness and minimum variance using Gaussian expectations and complete sufficient row means. Appendix~\ref{app:numerical} gives a stable numerical form.

At equal means, the biased estimator $\max(\bar R_1,\bar R_2)$ has MSE $\sigma^2/K$ for the same policy value. Along $J=K-1$, it has lower MSE than every globally unbiased estimator for sufficiently large $K$ (Figure~\ref{fig:gaussian-risk}). This pointwise comparison leaves the ordering at other mean pairs open.

\begin{proposition}[Risk of diagonal Holdout]
\label{prop:gaussian-holdout-risk}
Let $M=N=2$, $K\ge2$, and $R_{m,k}\sim\mathcal N(\mu_m,\sigma^2)$ independently, with common $\sigma>0$. Holdout uses no value or estimate of $\sigma$. At equal means, its MSE for $\Theta^{(K-1)}$ is
\begin{equation}
\label{eq:gaussian-holdout-risk}
\mathcal H_K=\frac{\sigma^2}{K}
+\frac{\sigma^2(K-1)}{\pi K\sqrt{2K-3}}.
\end{equation}
Use the loss and unbiased class from Theorem~\ref{thm:general-gaussian-risk}, restricted to $M=N=2$. As $K\to\infty$,
\begin{equation}
\label{eq:gaussian-unbiased-minimax}
\sup_{\boldsymbol\mu} L_K(\hbon;\boldsymbol\mu,\sigma)
\sim \inf_{U\in\mathcal U_K(\sigma)}\sup_{\boldsymbol\mu}L_K(U;\boldsymbol\mu,\sigma)
\sim\frac{\sigma^2}{\pi\sqrt{2K}}.
\end{equation}
For estimators unbiased at every $(\boldsymbol\mu,\sigma)$ without using $\sigma$, the minimax normalized risk $L_K/\sigma^2$, with supremum over all mean pairs and $\sigma>0$, is asymptotic to $1/(\pi\sqrt{2K})$, again attained by Holdout to first order. Allowing bias gives minimax rate $\sigma^2/K$ at fixed $\sigma$, or $1/K$ for the normalized unknown-variance loss.
\end{proposition}

Appendix~\ref{app:gaussian-holdout-risk} combines a uniform cross-moment bound with Theorem~\ref{thm:gaussian-risk} to prove the sharp constant; Appendix~\ref{app:gaussian-minimax} proves the unrestricted rate. Finite-$K$ minimaxity of Holdout and its exact worst-case mean gap remain unresolved by these results.

\begin{figure}[t]
\centering
\includegraphics[width=0.84\linewidth]{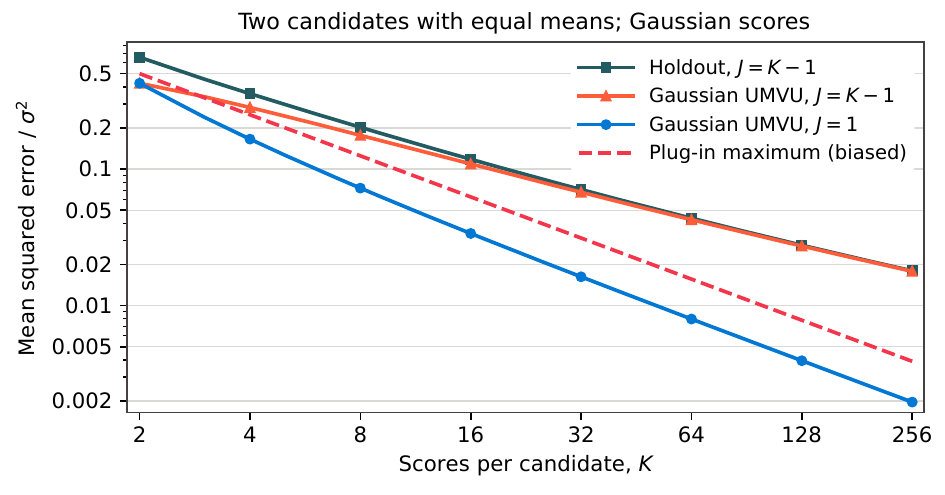}
\caption{Exact MSE divided by $\sigma^2$ for two independent Gaussian candidates at equal means. Holdout uses no variance input; the other unbiased curves assume known variance. The $J=1$ and $J=K-1$ policies have the same value here.}

\label{fig:gaussian-risk}
\end{figure}

\subsection{Selector depth and bounded scores}

\begin{corollary}[Selector depth and tied-mean risk]
\label{cor:gaussian-depth-regimes}
In the known-variance Gaussian pair model of Theorem~\ref{thm:gaussian-risk},
put $C(r)=\tfrac12+\pi^{-1}(\arcsin r+r/\sqrt{1-r^2})$.
The minimum unbiased risk at equal means satisfies
\[
\mathcal R_{J,K}\sim
\begin{cases}
\sigma^2/(2K),&J\text{ fixed},\\
\sigma^2C(r_0)/K,&J/K\to r_0\in[0,1),\\
\sigma^2/(\pi\sqrt{2Ks_K}),&s_K=K-J=o(K).
\end{cases}
\]
Thus bounded positive slack $s_K=O(1)$ gives order $K^{-1/2}$; for fixed
$s\ge1$, its leading constant is $1/(\pi\sqrt{2s})$.
These formulas give pointwise risk at equal means.
\end{corollary}

If $J/K$ stays bounded away from one, the Gaussian pair's minimum tied-mean
risk is $O(\sigma^2/K)$. Increasing $K$ at fixed $J$ evaluates one policy;
increasing $J$ also changes the policy. A uniform fixed-depth guarantee
holds beyond Gaussian scores when scores are bounded.

\begin{proposition}[Fixed-depth evaluation with bounded scores]
\label{prop:bounded-fixed-depth-risk}
Under Assumption~\ref{ass:indep}, suppose scores lie in $[a,a+B]$, $B>0$.
For any $M\ge N\ge2$ and $1\le J<K$, let
$\widehat\Theta^{\rm all}_{J,K}$ average the independent selection/evaluation
split over all $\binom KJ$ selection sets. Both it and the cyclic estimator
are unbiased for $\Theta^{(J)}$, and
\begin{align}
\label{eq:bounded-fixed-depth-risk}
\operatorname{MSE}(\widehat\Theta^{\rm all}_{J,K})
&\le\frac{(J+1)B^2}{4K},&
\operatorname{MSE}(\widehat\Theta_{J,K})
&\le\frac{B^2}{4}\min\left\{1,\frac{(J+1)^2}{K}\right\}.
\end{align}
The cyclic bound is uniform over $M,N$ and the candidate laws.
For fixed $M,N,J$, the globally unbiased minimax risk over this bounded
class has order $K^{-1}$: it lies between $B^2/(4MK)$ and
$(J+1)B^2/(4K)$.
\end{proposition}

Appendix~\ref{app:c7-depth-risk} proves the complete-average bound using its order-$(J+1)$ U-statistic structure \citep{hoeffding1948class}, and the cyclic bound using columnwise bounded differences. The cyclic bound has conservative constants and leaves minimum-variance optimality unresolved.

\section{Practical evaluation}
\label{sec:practical}
\subsection{Choosing an observation design}
Fix the deployed selector's $J$. If $J<K$ and a score matrix is available, cyclic splitting gives unbiased evaluation under Assumption~\ref{ass:indep}. For unbiased all-$K$ evaluation, keep $J=K$ and acquire a fresh winner score. With new acquisition, draw a uniform size-$N$ subset from the same pool, use $NJ$ selection calls, and score the winner once more: $NJ+1$ score calls. Unbiasedness requires that score's conditional mean to equal the winner's judge mean. This design retains subset and selection variation.

At budget $B$, use an independent pilot to compare designs for the same policy and target. With per-unit cost $c$, population variance $V$, and bias $b_0$, averaging about $B/c$ independent units has approximate MSE $b_0^2+cV/B$. Among unbiased designs, prefer smaller pilot estimates of $cV$, including generation and prompt overhead in $c$. A biased plug-in estimator also needs a justified estimate or bound for $b_0$.

\subsection{Population uncertainty}
\label{sec:inference}
For $T\ge2$ iid complete prompt/pool units, compute one estimate $H_p=\widehat\Theta_{J,K}^{(p)}$ per unit, keeping the policy and scoring protocol fixed. Under Assumption~\ref{ass:indep} and $\E[H_p^2]<\infty$, $\bar H=T^{-1}\sum_pH_p$ is unbiased for $\theta_{\rm pop}=\E_p[\Theta_p^{(J)}]$. With positive total variance, an asymptotic $1-\alpha$ interval is
\[
\bar H\ \pm\ z_{1-\alpha/2}\frac{s_H}{\sqrt T},
\qquad s_H^2=\frac1{T-1}\sum_{p=1}^T(H_p-\bar H)^2.
\]
Here $z_{1-\alpha/2}$ is the standard-normal quantile. The sample variance estimates
$\Var_p(\Theta_p^{(J)})+\E_p[\Var(H_p\mid\mathcal F_p)]$, including pool generation and judge noise. Iterated expectation and total variance give these moments; the iid central limit theorem gives the interval's asymptotic coverage. If every score lies in $[a,b]$, then $H_p$ also lies in that interval, and Hoeffding's inequality gives a finite-sample interval with half-width $(b-a)\sqrt{\log(2/\alpha)/(2T)}$.

Folds, candidates, and score calls within one pool cannot serve as independent prompt/pool replicates. When several pools share a sampled prompt, average within each prompt using prespecified nonnegative weights summing to one, then treat the independent prompt averages as the units. For a paired policy comparison, apply the variance formula to per-prompt differences. Resample whole units.

\subsection{Sensitivity to dependent scores}
\label{sec:dependence}
Fix the pool and its intended independent-score target $\Theta^{(J)}$. The observed matrix may now have any integrable joint law. Write $q_{m,t}=q_m(a_t)$ for candidate $m$'s subset-and-tie-averaged selection weight, and define
\[
d_{m,t}(a_t)=\E[v_{m,t}\mid\mathcal F,a_t]-\mu_m,
\quad \bar p_m=\frac1K\sum_t\E[q_{m,t}\mid\mathcal F],
\quad p_m^J=\Pr\{W_J=m\mid\mathcal F\}.
\]
Conditioning on the keys gives the exact identity
\[
\E[\widehat\Theta_{J,K}\mid\mathcal F]-\Theta^{(J)}
=\underbrace{\frac1K\sum_t\E\!\left[\sum_mq_{m,t}d_{m,t}\mid\mathcal F\right]}_{\text{evaluation conditional means}}
+\underbrace{\sum_m\mu_m(\bar p_m-p_m^J)}_{\text{selector law}}.
\]
Consequently the absolute bias is at most $C+D\operatorname{TV}(\bar p,p^J)$, where
$C=K^{-1}\sum_t\E[\max_m|d_{m,t}|\mid\mathcal F]$,
$D=\max_m\mu_m-\min_m\mu_m$, and
$\operatorname{TV}(p,q)=\tfrac12\sum_m|p_m-q_m|$.
The first bound uses the simplex weights; the second uses the range-$D$ total-variation inequality. An observed correlation or a single score matrix cannot certify these sensitivity parameters.

For example, consider $R_{m,k}=\mu_m+\sigma(\sqrt\rho\,U_m+\sqrt{1-\rho}\,Z_{m,k})$, where $0\le\rho<1$ and $U,Z$ are independent standard normals drawn afresh for each matrix, independently of $\mathcal F$. Selection keys have variance $\sigma^2[\rho+(1-\rho)/J]$ and
\[
d_{m,t}(a_t)=\frac{J\rho}{1+(J-1)\rho}(a_{m,t}-\mu_m).
\]
This dependence changes both evaluation conditional means and the selection noise law while preserving each score's marginal mean and variance. A common additive shift $b_k$ shared by every candidate in column $k$ preserves rankings and changes cyclic evaluation by exactly $K^{-1}\sum_k b_k$, so zero-average common shifts cancel. Both properties can fail under arbitrary candidate-specific drift.

\section{Experiments}
\label{sec:experiments}
\label{sec:designs}
All experiments use synthetic scores with known laws and targets to check the implementation and compare evaluation designs.

\subsection{Independent-score checks}
Five fixed Bernoulli pools with exact binomial targets test tie handling; independent Gaussian controls at $M=16,N=4$ compare reuse and Holdout at the same selector depth. All five Bernoulli and four Gaussian Holdout differences lie within $1.81$ matching Monte Carlo standard errors of zero. Appendix~\ref{app:numerical} gives the protocols, data, and code.

\subsection{Error and cost for a fixed selector}
For fixed $J=1$, we use $M=N=8$, iid $\mathrm{Unif}(0.2,0.8)^8$ mean pools, Gaussian score SD $0.6$, and $20{,}000$ pools per $K$. Cyclic evaluation uses the full matrix of $8K$ scores. The fixed split selects with the first column and evaluates the winner on the remaining $K-1$ columns; acquiring scores only for these steps would cost $8+K-1$ calls. Costs exclude generation and prompt overhead, which can change the budget ranking. Cyclic evaluation has lower pool-specific RMSE; the fixed split has lower total population variance at equal score-call budget (Table~\ref{tab:one-score}).

\begin{table}[ht]
\centering
\caption{Full-matrix cyclic and adaptive fixed-split evaluation of the one-score target $\Theta^{(1)}$. The last column compares population variance at equal score-call budget, excluding generation and prompt overhead; values below one favor the fixed split.}
\label{tab:one-score}
\begin{tabular}{r|rrr|r}
\toprule
& \multicolumn{3}{c|}{RMSE against $\Theta^{(1)}$} & $c_FV_F/(c_HV_H)$ \\
$K$ & Cyclic split ($H$) & Fixed split ($F$) & Score reuse & \\
\midrule
$4$ & $0.240$ & $0.379$ & $0.845$ & $0.824$ \\
$8$ & $0.138$ & $0.271$ & $0.831$ & $0.793$ \\
\bottomrule
\end{tabular}
\end{table}

\subsection{Finite-sample comparison beyond two candidates}
\label{sec:finite-experiment}
We compare Holdout and the biased row-mean comparator $P=\E_S[\max_{m\in S}\bar R_m]$ against the same $\Theta^{(K-1)}$. Table~\ref{tab:general-finite-risk} shows the prespecified $K=16$ slice of a $27$-cell Gaussian study with variance one. The mean vectors are all zero, $(0,0,-1,-2,\ldots)$, and $(0,-1,-2,\ldots)$. Appendix~\ref{app:numerical} gives the full grid, simulation protocol, target-integration checks, and saved outputs.

\begin{table}[H]
\centering
\caption{MSE for $\Theta^{(15)}$ with fixed means, $K=16$, and $\sigma=1$. Parentheses give the Monte Carlo standard error of the paired difference $\Delta=\mathrm{MSE}(H)-\mathrm{MSE}(P)$; positive $\Delta$ favors $P$.}
\label{tab:general-finite-risk}
\begin{tabular}{rrlrrr}
\toprule
$M$ & $N$ & Means & Holdout $H$ & Plug-in $P$ & $\Delta$ (MCSE) \\
\midrule
$4$ & $2$ & All tied & $0.0366$ & $0.0391$ & $-0.00255\;(0.00026)$ \\
$4$ & $2$ & Two tied best & $0.0262$ & $0.0244$ & $+0.00180\;(0.00006)$ \\
$4$ & $2$ & Separated & $0.0247$ & $0.0243$ & $+0.00034\;(0.00002)$ \\
$8$ & $4$ & All tied & $0.0452$ & $0.0798$ & $-0.03453\;(0.00039)$ \\
$8$ & $4$ & Two tied best & $0.0250$ & $0.0223$ & $+0.00265\;(0.00008)$ \\
$8$ & $4$ & Separated & $0.0226$ & $0.0222$ & $+0.00032\;(0.00002)$ \\
$8$ & $8$ & All tied & $0.1821$ & $0.1496$ & $+0.03252\;(0.00098)$ \\
$8$ & $8$ & Two tied best & $0.1181$ & $0.0626$ & $+0.05542\;(0.00065)$ \\
$8$ & $8$ & Separated & $0.0669$ & $0.0623$ & $+0.00469\;(0.00026)$ \\
\bottomrule
\end{tabular}
\end{table}

At $K=16$ with all means tied, Holdout has lower MSE for $(M,N)=(4,2)$ and $(8,4)$, while the plug-in is better for $(8,8)$. The plug-in is better in all nine cells with two tied best candidates. The full grid shows reversals with depth for some all-tied pools and agreement to numerical precision on the saved $K=64$ draws for separated means. Finite-sample rankings depend on the mean vector, subset size, and depth.

\subsection{Dependence stress test}
\label{sec:dependence-experiment}
We use the equicorrelated model from Section~\ref{sec:dependence} with $M=16,N=4,K=8,\sigma=0.4$ and $40{,}000$ iid $\operatorname{Beta}(2,2)^{16}$ mean pools, sharing primitive shocks across $\rho\in\{0,0.1,0.4,0.8\}$. Scoring the realized dependent selections by their true means gives $B_{\rm actual}$; doing the same for independent-law selections gives $B_{\rm independent}$. The paired differences $H-B_{\rm actual}$ and $B_{\rm actual}-B_{\rm independent}$ estimate evaluation and selector-law bias against the independent-score target. They sum to $H-B_{\rm independent}$ in each pool.

\begin{table}[ht]
\centering
\caption{Bias against the independent-score target at each $\rho$. Entries are paired mean differences over $40{,}000$ pools, with Monte Carlo standard errors in parentheses.}
\label{tab:dependence}
\begin{tabular}{rrrr}
\toprule
$\rho$ & Evaluation & Selector law & Total bias \\
\midrule
$0.0$ & $-0.0001\;(0.0003)$ & $+0.0000\;(0.0000)$ & $-0.0001\;(0.0003)$ \\
$0.1$ & $+0.0551\;(0.0004)$ & $-0.0160\;(0.0001)$ & $+0.0391\;(0.0004)$ \\
$0.4$ & $+0.1832\;(0.0005)$ & $-0.0478\;(0.0002)$ & $+0.1353\;(0.0005)$ \\
$0.8$ & $+0.3074\;(0.0006)$ & $-0.0720\;(0.0003)$ & $+0.2355\;(0.0006)$ \\
\bottomrule
\end{tabular}
\end{table}
The negative selector-law component partially offsets the evaluation bias in this experiment.

\section{Related work}
Classical unbiased functional estimation and U-statistics provide the independent-kernel construction and its minimum-variance symmetrization \citep{halmos1946unbiased,hoeffding1948class}. A $J$-score selection followed by one independent reward is an order-$(J+1)$ functional in each candidate law, which proves achievability. Our Bernoulli restriction proves that BoN requires this degree, and the Gaussian analysis quantifies the cost of exact unbiasedness near the boundary.

The subset weights extend the binary pass@$k$ counting identity \citep{chen2021codex}. Cross-fitting separates evaluation from fitting or selection \citep{chernozhukov2018double}; \citet{zrnic2024fithian} study inference after choosing a noisy winner. Our target averages the judge mean of a specified selection rule over its score noise, subsets, and ties. Population uncertainty uses complete prompt/pool units, consistent with the evaluation-design concerns in \citet{miller2024error}.

Reward overoptimization also includes systematic disagreement with external quality \citep{gao2023scaling}. Ensembles \citep{coste2023reward,eisenstein2023helping}, Soft BoN \citep{verdun2024softbon}, BoP / HedgeTune \citep{khalaf2025inference}, and inference-time pessimism \citep{huang2025inferencetimepessimism} modify scoring or selection. We study evaluation of the specified sample-mean selector under a stochastic scalar judge, as can arise with LLM judges \citep{zheng2023judging,zhuge2024judges} or sampled reward evaluation \citep{ziegler2019fine,ouyang2022instructgpt}.

\section{Limitations and conclusion}
\label{sec:discussion}
The unbiasedness guarantee requires independent scores with stable candidate-specific laws. Randomized call order, separate sessions, and replicated batches can expose temporal or shared-call effects. Independence remains an assumption.

Dependence can change both the winner law and the held-out conditional means; Section~\ref{sec:dependence} bounds their contributions. A fresh-score comparison on realized winners probes evaluation error. Selector-law bias remains unmeasured by that comparison. The experiments use synthetic scores; whether a real judge satisfies the model remains untested here. Judge means can differ from human quality, and pairwise or vector feedback requires a separate target.

The risk separation holds for fixed $M\ge N\ge2$ with common-variance Gaussian scores and $J=K-1$; its constants may depend on $M,N$. Boundary-regime rates for growing pools, unequal variances, and other score laws lie outside this analysis. The distribution-free bounded-score upper bound is uniform in pool size; the matching globally unbiased minimax rate assumes fixed $M,N,J$. The sharp constant is proved for two candidates. In that Gaussian model, fixed-$J$ minimum unbiased risk at equal means has order $K^{-1}$, and the unbiased minimax risk along $J=K-1$ has order $K^{-1/2}$.

From a fixed matrix, exact unbiased evaluation of a $J$-score BoN policy is possible precisely when $J<K$ over the stated distribution class. Holdout attains the optimal Gaussian unbiased rate at that boundary without knowing the variance. Allowing bias yields a faster minimax rate. Evaluation design therefore starts with the policy to be measured, then weighs acquisition cost and mean-squared error against the requirement of exact unbiasedness.

\bibliographystyle{plainnat}
\bibliography{refs}
\appendix
\section{Proof of Theorem~\ref{thm:frontier} (estimability frontier)}
\label{app:frontier}

\begin{proof}
Condition on the realized pool. For $J<K$, selection with $J$ columns and independent evaluation with the remaining columns is unbiased for $\Theta^{(J)}$; subset, tie, and cyclic averaging preserve its expectation (Theorem~\ref{thm:unbiased}).

For impossibility, first take $M=N=2$ and independent Bernoulli scores with means $(\mu_1,\mu_2)\in(0,1)^2$. If $c_m(r)=\sum_{k=1}^K r_{m,k}$, the expectation of any integrable statistic is
\[
\E_{\boldsymbol\mu}T
=\sum_{r\in\{0,1\}^{2\times K}}T(r)
\prod_{m=1}^2\mu_m^{c_m(r)}(1-\mu_m)^{K-c_m(r)}.
\]
This is a polynomial of degree at most $K$ in each mean. Integrability at any interior pair makes every coefficient $T(r)$ finite, since every binary matrix has positive probability there.

For the $J$-score selector, let $C'_m\sim\operatorname{Binom}(J,\mu_m)$ independently. Uniform ties give
\begin{equation}
\label{eq:thetaJ-binary}
\Theta^{(J)}_{\rm bin}
=\mu_1\Pr(C'_1>C'_2)+\mu_2\Pr(C'_1<C'_2)
+\frac{\mu_1+\mu_2}{2}\Pr(C'_1=C'_2).
\end{equation}
Each probability has degree at most $J$ in each mean, so the target has degree at most $J+1$ in each mean. At $\mu_2=0$, its polynomial extension is
\[
\Theta^{(J)}_{\rm bin}(\mu_1,0)
=\mu_1\left[1-\frac12(1-\mu_1)^J\right].
\]
The coefficient of $\mu_1^{J+1}$ is $(-1)^{J+1}/2\ne0$. Thus the target has degree exactly $J+1$ in $\mu_1$, and by symmetry in $\mu_2$. Equality on the open square would imply a polynomial identity everywhere, allowing evaluation at the boundary without an unbiasedness assumption there. Consequently no statistic of the $2K$ observed scores can be unbiased for this target when $J\ge K$.

For any fixed $M\ge N\ge2$, retain these two Bernoulli candidates and give the remaining $M-2$ candidates the constant score $L<0$. These laws satisfy Assumption~\ref{ass:indep}. Any active candidate beats every constant one. Counting subsets with zero, one, or two active candidates gives
\begin{equation}
\label{eq:embedding}
\Theta^{(J)}(M,N,\boldsymbol\mu_{\rm emb})
=\frac{\binom{M-2}{N}L+\binom{M-2}{N-1}(\mu_1+\mu_2)
+\binom{M-2}{N-2}\Theta^{(J)}_{\rm bin}(\mu_1,\mu_2)}{\binom MN},
\end{equation}
with $\binom ab=0$ for $b<0$ or $b>a$. The binary term's coefficient is positive, while the other terms have degree at most one. The target therefore retains degree $J+1$ in each active mean. The constant rows add no random observations, so every statistic's expectation still has degree at most $K$ in each active mean. This contradicts unbiasedness for every $J\ge K$ and proves the frontier for all $M\ge N\ge2$.
\end{proof}

Averaging out independent parameter-free auxiliary randomization gives a statistic of the matrix with the same expectation, so the obstruction also applies to randomized estimators. The same degree argument applies on any nondegenerate bounded score interval by affinely mapping the active Bernoulli scores and the lower constant into that interval.

\subsection{The frontier with known common Gaussian noise}
\label{app:gaussian-frontier}

\begin{proposition}[Known Gaussian two-candidate frontier]
\label{prop:gaussian-frontier}
Fix $M=N=2$, integers $K\ge2$ and $J\ge1$, and a known $\sigma>0$. Let the $2K$ observed scores be mutually independent with $R_{m,k}\sim\mathcal N(\mu_m,\sigma^2)$. The $J$-score selector uses fresh independent draws from these same laws. There is an estimator $T(R)$ that is integrable and satisfies $\E_{\mu_1,\mu_2}T=\Theta^{(J)}(\mu_1,\mu_2)$ for every $(\mu_1,\mu_2)\in\mathbb R^2$ if and only if $J<K$.
\end{proposition}

\begin{proof}
For $J<K$, select with $J$ columns and evaluate with an independent remaining column. For impossibility, restrict to $\mu_1=d/2$, $\mu_2=-d/2$, with $d\in\mathbb R$. Writing $a=\sqrt J/(2\sigma)$, the target on this submodel is
\[
\theta_J(d)=\frac d2\operatorname{erf}(ad),
\qquad
\operatorname{erf}(z)=\frac2{\sqrt\pi}\int_0^z e^{-u^2}\,du.
\]
Indeed, candidate 1 wins with probability $\Phi(d\sqrt{J/(2\sigma^2)})$.

The observed mean difference $D=\bar R_1-\bar R_2$ has law $\mathcal N(d,\tau^2)$, where $\tau^2=2\sigma^2/K$. The Gaussian residual vector after projecting the scores onto this difference is independent of $D$ and has a law independent of $d$. Conditioning $T$ on $D$ therefore gives a single parameter-independent function $t$ with $\E_d|t(D)|<\infty$ and $\E_dt(D)=\theta_J(d)$ for every real $d$.

Let $\phi_\tau$ be the density of $\mathcal N(0,\tau^2)$ and define
\[
H(z)=\int_{\mathbb R}t(x)\phi_\tau(x)e^{xz/\tau^2}\,dx,
\qquad
F(z)=e^{-z^2/(2\tau^2)}H(z).
\]
Both functions are entire. To justify this under integrability alone, on every compact set with $|\Re z|\le A$, the integrand is dominated by
$|t(x)|\phi_\tau(x)(e^{Ax/\tau^2}+e^{-Ax/\tau^2})$, whose integral is finite by integrability at means $\pm A$. Truncated integrals are entire and converge locally uniformly. For real $d$, the Gaussian density ratio gives $F(d)=\E_dt(D)=\theta_J(d)$, so the identity theorem gives $F(z)=(z/2)\operatorname{erf}(az)$ for all complex $z$.

Since $t\phi_\tau\in L^1(\mathbb R)$, the Riemann--Lebesgue lemma requires $H(iy)\to0$ as $y\to+\infty$. Writing $\operatorname{erfi}(u)=(2/\sqrt\pi)\int_0^u e^{v^2}\,dv$,
\[
H(iy)=-\frac y2e^{-y^2/(2\tau^2)}\operatorname{erfi}(ay)
\sim-\frac{1}{2a\sqrt\pi}
\exp\!\left(\frac{J-K}{4\sigma^2}y^2\right).
\]
Here $\operatorname{erfi}(u)\sim e^{u^2}/(\sqrt\pi u)$ follows, for example, by l'Hopital's rule. For $J=K$ the limit is $-\sigma/\sqrt{\pi K}\ne0$; for $J>K$ the magnitude diverges. Both contradict the required Fourier decay.
\end{proof}

The result applies to every measurable estimator with finite absolute expectation, including estimators with infinite variance. Averaging out independent parameter-free auxiliary randomization preserves expectation, so the proof also covers randomized estimators. The result concerns exact unbiasedness for known common Gaussian noise at $M=N=2$.

\section{Proof of Theorem~\ref{thm:gaussian-risk}}
\label{app:gaussian-risk}

Write $s=(\mu_1+\mu_2)/2$, $d=\mu_1-\mu_2$, and $a=\sqrt J/(2\sigma)$.
The independent variables $S\sim\mathcal N(s,\tau^2/4)$ and $D\sim\mathcal N(d,\tau^2)$ satisfy
$\Theta^{(J)}=s+(d/2)\operatorname{erf}(ad)$.
For $b=a/\sqrt{1-J/K}$, Gaussian convolution and integration by parts give
\begin{align*}
\E_d[\operatorname{erf}(bD)]
&=\operatorname{erf}\!\left(\frac{bd}{\sqrt{1+2b^2\tau^2}}\right)
=\operatorname{erf}(ad),\\
\E_d[D\operatorname{erf}(bD)]
&=d\operatorname{erf}(ad)
+\frac{2\tau^2b}{\sqrt\pi}\E_d[e^{-b^2D^2}].
\end{align*}
Substitution in Equation~\eqref{eq:gaussian-umvu} proves unbiasedness.
Also $|T_{J,K}|\le |S|+|D|/2+\tau^2b/\sqrt\pi$, so its variance is finite at every mean pair.

The row means $Y=(\bar R_1,\bar R_2)$ are complete and sufficient: their known-variance normal location family is a full exponential family with natural-parameter space $\mathbb R^2$.
For any statistic $U$ integrable and unbiased at every mean pair, sufficiency makes $\E[U\mid Y]$ parameter-independent, and completeness identifies it with $T_{J,K}$.
When $U$ has finite variance,
\[
\Var(U)=\Var(T_{J,K})+\E[\Var(U\mid Y)]\ge\Var(T_{J,K}).
\]
Equality requires $U=T_{J,K}$ almost surely. The normal laws at different means are mutually absolutely continuous, giving uniqueness up to common null sets. Independent parameter-free randomization cannot improve this bound.

For the tied-mean risk, put $X=D/\tau\sim\mathcal N(0,1)$, $q=b\tau$, and
\begin{gather*}
C=\E[e^{-2q^2X^2}]=(1+4q^2)^{-1/2},\\
B=\E[X\operatorname{erf}(qX)e^{-q^2X^2}],\qquad
Q=\E[\operatorname{erf}(qX)^2].
\end{gather*}
Gaussian integration by parts gives
\[
(1+2q^2)B=\frac{2qC}{\sqrt\pi},
\qquad
\frac{dQ}{dq}=\frac{4B}{\sqrt\pi}
=\frac{8q}{\pi(1+2q^2)\sqrt{1+4q^2}}.
\]
Since $Q(0)=0$, integration yields $Q=(2/\pi)\arcsin r$, where
$r=2q^2/(1+2q^2)=J/K$.
Applying $\E[X^2f(X)]=\E[f(X)]+\E[f''(X)]$ to $f(x)=\operatorname{erf}(qx)^2$ also gives
\[
A:=\E[X^2\operatorname{erf}(qX)^2]
=Q+\frac{8q^2C}{\pi}-\frac{8q^3B}{\sqrt\pi}
=Q+\frac{8q^2C}{\pi(1+2q^2)}.
\]
Thus the mean-zero term $T_{J,K}-S$ has second moment
\begin{align*}
\E[(T_{J,K}-S)^2]
&=\tau^2\left(\frac A4-\frac{qB}{\sqrt\pi}+\frac{q^2C}{\pi}\right)\\
&=\frac{\tau^2}{2\pi}\left(\arcsin r+\frac r{\sqrt{1-r^2}}\right),
\end{align*}
using $q^2C=r/(2\sqrt{1-r^2})$.
Add the independent variance $\Var(S)=\tau^2/4$ to obtain Equation~\eqref{eq:gaussian-tie-risk}.
At $J=K-1$,
\[
\frac r{\sqrt{1-r^2}}=\frac{K-1}{\sqrt{2K-1}}
=\sqrt{K/2}+O(K^{-1/2}),\qquad
\arcsin r=\frac\pi2+O(K^{-1/2}),
\]
so $\mathcal R_{K-1,K}=\sigma^2/(\pi\sqrt{2K})+\sigma^2/K+O(\sigma^2K^{-3/2})$.
For fixed $J$, $r\to0$ and $\mathcal R_{J,K}\sim\sigma^2/(2K)$.

Finally, $\max(\bar R_1,\bar R_2)=S+|D|/2$.
At $\mu_1=\mu_2=s$, its bias is $\sigma/\sqrt{\pi K}$ and its MSE against $\Theta^{(J)}=s$ is
\[
\E\!\left[\left(\max(\bar R_1,\bar R_2)-s\right)^2\right]
=\Var(S)+\frac14\E[D^2]
=\frac{\sigma^2}{K}.
\]
This pointwise comparison leaves the ordering at other mean pairs open.

\subsection{Unrestricted minimax rate}
\label{app:gaussian-minimax}
Fix $J=K-1$ in the two-candidate Gaussian model. Let $\Phi$ denote the standard normal CDF.
For unrestricted estimation, put $P=\max(\bar R_1,\bar R_2)$ and $m=\max(\mu_1,\mu_2)$.
The maximum is $1$-Lipschitz in Euclidean norm, so $\E[(P-m)^2]\le2\sigma^2/K$.
The difference between the oracle mean and the target is
\[
0\le m-\Theta^{(J)}
=|d|\Phi\!\left(-\frac{|d|\sqrt J}{\sqrt2\sigma}\right)
\le\frac{\sigma}{\sqrt{2eJ}}.
\]
Therefore, at $J=K-1$,
\[
\sup_{\mu_1,\mu_2}\E[(P-\Theta^{(K-1)})^2]
\le\frac{4\sigma^2}{K}+\frac{\sigma^2}{e(K-1)}
\le\left(4+\frac2e\right)\frac{\sigma^2}{K}.
\]
For the unrestricted lower bound, restrict to $\mu_1=\mu_2=s$, where the target is $s$ and the data are $2K$ independent $\mathcal N(s,\sigma^2)$ observations.
A prior $s\sim\mathcal N(0,L^2)$ gives posterior variance $(L^{-2}+2K/\sigma^2)^{-1}$, a lower bound on every estimator's worst-case risk.
Letting $L\to\infty$ gives $\sigma^2/(2K)$, completing the unrestricted minimax rate.
The row-mean maximum uses no $\sigma$, and these bounds hold at every scale. Dividing by $\sigma^2$ proves the unknown-variance normalized rate as well. These constants bound the unrestricted minimax risk.

\section{Proof of Proposition~\ref{prop:gaussian-holdout-risk}}
\label{app:gaussian-holdout-risk}

Write $s=(\mu_1+\mu_2)/2$, $d=\mu_1-\mu_2$, and
$\delta=d/(\sigma\sqrt2)$.
For each column define
\[
A_k=(R_{1,k}+R_{2,k})/2,\qquad
Z_k=(R_{1,k}-R_{2,k}-d)/(\sigma\sqrt2).
\]
The $Z_k$ are iid standard normals, the $A_k$ are iid
$\mathcal N(s,\sigma^2/2)$, and the two collections are independent.
Put
\[
S_k=\operatorname{sign}\!\left(\sum_{j\ne k}Z_j+(K-1)\delta\right),
\qquad
\bar S=K^{-1}\sum_k S_k,\qquad
\bar A=K^{-1}\sum_k A_k.
\]
The selection keys have continuous distributions, so the value assigned to $\operatorname{sign}(0)$ does not affect the result.
The winner's held-out score gives the exact representation
\[
\hbon=\bar A+\frac d2\bar S+V,\qquad
V=\frac{\sigma}{\sqrt2 K}\sum_{k=1}^K Z_kS_k.
\]
Since $Z_k$ is independent of $S_k$, $\E V=0$.
Every $S_k$ has mean
$m=2\Phi(\delta\sqrt{K-1})-1$, and
$\Theta^{(K-1)}=s+dm/2$.
Consequently $\hbon$ is unbiased at every mean pair.
Here $\Phi$ and $\phi$ denote the standard-normal CDF and density.

\paragraph{Exact noise second moment.}
For distinct indices, let $W=\sum_{j=3}^KZ_j$ and
$a=(K-1)\delta$.
Conditional on $W$, the two factors below are independent.
The identity $\E[Z\operatorname{sign}(Z+t)]=2\phi(t)$ gives
\begin{align*}
\E[Z_1Z_2S_1S_2]
&=\E\!\left[
  \E[Z_1\operatorname{sign}(Z_1+W+a)\mid W]\,
  \E[Z_2\operatorname{sign}(Z_2+W+a)\mid W]\right]\\
&=4\E[\phi(W+a)^2]
=\frac{2}{\pi\sqrt{2K-3}}
  \exp\!\left(-\frac{a^2}{2K-3}\right).
\end{align*}
For $K=2$, $W=0$ deterministically and the same formula holds.
The diagonal second moments are $\E[Z_k^2S_k^2]=1$, so
\begin{equation}
\label{eq:holdout-noise-second-moment}
\Var(V)=\E[V^2]
=\frac{\sigma^2}{2K}
+\frac{\sigma^2(K-1)}{\pi K\sqrt{2K-3}}
  \exp\!\left(-\frac{(K-1)^2\delta^2}{2K-3}\right).
\end{equation}
At $d=0$, the signal term vanishes.
Adding the independent variance $\Var(\bar A)=\sigma^2/(2K)$ proves Equation~\eqref{eq:gaussian-holdout-risk}.
In particular, $\mathcal H_K\sim\sigma^2/(\pi\sqrt{2K})$.

\paragraph{Uniform bound.}
Put $B=(d/2)(\bar S-m)$.
Jensen's inequality gives
$\Var(\bar S)\le K^{-1}\sum_k\E[(S_k-m)^2]=1-m^2$.
With $t=|d|\sqrt{K-1}/(\sigma\sqrt2)$,
\[
\Var(B)\le d^2\Phi(t)\Phi(-t)
\le d^2\Phi(-t)
\le\frac{d^2}{2}\exp\!\left(-\frac{d^2(K-1)}{4\sigma^2}\right)
\le\frac{2\sigma^2}{e(K-1)}.
\]
The penultimate inequality uses
$\Phi(-t)=\int_0^\infty\phi(t+u)\,du\le\tfrac12e^{-t^2/2}$ for $t\ge0$.
Define the deterministic bounds
\[
\alpha_K=\frac{\sigma^2}{2K}
+\frac{\sigma^2(K-1)}{\pi K\sqrt{2K-3}},
\qquad
\beta_K=\frac{2\sigma^2}{e(K-1)}.
\]
Equation~\eqref{eq:holdout-noise-second-moment} gives $\Var(V)\le\alpha_K$.
Cauchy--Schwarz applies even when the centered terms $B,V$ are dependent and gives
\begin{equation}
\label{eq:holdout-uniform-risk}
\sup_{\boldsymbol\mu}
\E_{\boldsymbol\mu,\sigma}[(\hbon-\Theta^{(K-1)})^2]
\le
\frac{\sigma^2}{2K}
+(\sqrt{\alpha_K}+\sqrt{\beta_K})^2.
\end{equation}
Since $\alpha_K\sim\sigma^2/(\pi\sqrt{2K})$ and
$\beta_K=O(\sigma^2/K)$, the right-hand side is
$\sigma^2/(\pi\sqrt{2K})+O(\sigma^2K^{-3/4})$.
The exact tied risk supplies the matching lower bound for the supremum.
Finally, Theorem~\ref{thm:gaussian-risk} lower-bounds every globally unbiased estimator's worst-case MSE by its minimum possible tied-mean MSE
$\mathcal R_{K-1,K}\sim\sigma^2/(\pi\sqrt{2K})$.
Diagonal Holdout belongs to this unbiased class, so the bounds prove the stated minimax asymptotic.
The argument does not require the worst-case means to be tied.

\paragraph{Unknown common variance.}
The estimator $\hbon$ uses no value or estimate of $\sigma$ and is unbiased for every pair $(\boldsymbol\mu,\sigma)$.
If $\mathcal U_K^{\rm unknown}$ denotes this globally unbiased class with $\sigma$ unknown, the scale-normalized statement is
\[
\inf_{U\in\mathcal U_K^{\rm unknown}}
\sup_{\boldsymbol\mu\in\mathbb R^2,\,\sigma>0}
\frac{\E_{\boldsymbol\mu,\sigma}
[(U-\Theta^{(K-1)}(\boldsymbol\mu,\sigma))^2]}{\sigma^2}
\sim\frac1{\pi\sqrt{2K}},
\]
attained to first order by $\hbon$.
For the lower bound restrict to $\sigma=1$; the uniform upper bound above applies at every scale.
An unnormalized supremum over all $\sigma>0$ is infinite.
If bias is allowed, the observed row-mean maximum also uses no $\sigma$ and the bounds in Appendix~\ref{app:gaussian-minimax} continue to give normalized minimax rate $K^{-1}$.

\section{Proof of Theorem~\ref{thm:general-gaussian-risk}}
\label{app:general-gaussian-risk}

For this appendix put $J=K-1$ and define
\[
c=\frac{N(N-1)}{M(M-1)},\qquad
a=\frac{N(M-N)}{M(M-1)},\qquad b=1-2a-c.
\]
These are, respectively, the probabilities that a uniform size-$N$ subset
contains both of two designated candidates, contains the first but not the
second, and contains neither. The second but not the first also has
probability $a$. Equivalently,
$c=\binom{M-2}{N-2}/\binom MN$,
$a=\binom{M-2}{N-1}/\binom MN$, and
$b=\binom{M-2}{N}/\binom MN$, with out-of-range binomial coefficients zero.
Let $\mathcal R_{J,K}$ be the two-candidate Gaussian minimum unbiased tied-mean
risk in Equation~\eqref{eq:gaussian-tie-risk}. Set
\begin{align*}
B_{M,N,K}
&=\frac{(c+2a)^2\sigma^2}{2K}
 +c^2\left(\mathcal R_{K-1,K}-\frac{\sigma^2}{2K}\right),\\
\alpha_{N,K}
&=\frac{\sigma^2}{K}
 +\frac{N^2(N-1)\sigma^2(K-1)}{4\pi K\sqrt{2K-3}},
\qquad
\beta_{N,K}=\frac{(N-1)\sigma^2}{K-1}.
\end{align*}
We prove the finite-$K$ inequalities
\begin{equation}
\label{eq:general-gaussian-unbiased-bounds}
B_{M,N,K}
\le\inf_{U\in\mathcal U_K(\sigma)}\sup_{\boldsymbol\mu}L_K(U;\boldsymbol\mu,\sigma)
\le\sup_{\boldsymbol\mu}L_K(\hbon;\boldsymbol\mu,\sigma)
\le(\sqrt{\alpha_{N,K}}+\sqrt{\beta_{N,K}})^2.
\end{equation}
Here $B_{M,N,K}\ge c^2\mathcal R_{K-1,K}$ and
$c>0$. The lower bound is therefore asymptotic to
$c^2\sigma^2/(\pi\sqrt{2K})$ up to lower-order terms, while the upper bound
is $O_N(\sigma^2/\sqrt K)$.

\paragraph{Uniform upper bound for actual Holdout.}
First take $M=N$. Let $W_k$ be the candidate with the largest leave-one-out mean,
write $\varepsilon_{m,k}=R_{m,k}-\mu_m$, and choose a fixed candidate
$m_\star$ with largest mean $\mu_\star$. Gaussian keys tie with probability zero.
With $\Delta_m=\mu_\star-\mu_m$, write
\[
\hbon=\mu_\star+A-D,\qquad
A=\frac1K\sum_k\varepsilon_{W_k,k},\qquad
D=\frac1K\sum_k\Delta_{W_k}.
\]
Every held-out noise term has mean zero and variance $\sigma^2$ because
its selection is independent of the held-out column.
For the signal term, let $e_m$ denote candidate $m$'s centered selection
mean in one block. If $W$ wins, then
$0\le\Delta_W\le(e_W-e_{m_\star})_+$, and hence
\[
\E[\Delta_W^2]
\le\sum_{m\ne m_\star}\E[(e_m-e_{m_\star})_+^2]
=\frac{(N-1)\sigma^2}{K-1}.
\]
The equality uses centered normal differences of variance
$2\sigma^2/(K-1)$. Jensen's inequality across the $K$ blocks gives
$\Var(D)\le\E[D^2]\le\beta_{N,K}$.

To bound the cross-moments of the held-out noise, condition on
$T_m=\sum_{k=3}^K R_{m,k}$ and put $\nu_m=T_m+\mu_m$.
For $K=2$, the empty sum is zero. Let $Z_m=\nu_m+\varepsilon_m$, where the
$\varepsilon_m$ are independent $\mathcal N(0,\sigma^2)$, and define
\[
C_{mn}(\boldsymbol\nu)
=\E\!\left[\varepsilon_m\,
\mathbf1\{n=\arg\max_j Z_j\}\,\middle|\,T\right].
\]
Conditional independence of columns 1 and 2 implies
\[
\E[\varepsilon_{W_1,1}\varepsilon_{W_2,2}\mid T]
=\sum_{m,n}C_{mn}C_{nm}.
\]
Let $f_m,F_m$ be the density and CDF of $\mathcal N(\nu_m,\sigma^2)$.
For $m\ne n$, conditioning on the winning value $Z_n=x$ and integrating
the truncated centered normal gives
\[
C_{mn}
=-\sigma^2\int_{\mathbb R} f_m(x)f_n(x)
           \prod_{\ell\ne m,n}F_\ell(x)\,dx.
\]
This is symmetric in $m,n$. Since each row of $C$ sums to zero,
$C_{mm}=-\sum_{n\ne m}C_{mn}$. Cauchy--Schwarz applied to these row sums yields
\[
\sum_{m,n}C_{mn}C_{nm}=\|C\|_F^2
\le N\sum_{m\ne n}C_{mn}^2.
\]
Dropping the CDF factors and integrating the product of two normal densities,
\[
|C_{mn}|\le\frac{\sigma}{2\sqrt\pi}
\exp\!\left(-\frac{(\nu_m-\nu_n)^2}{4\sigma^2}\right).
\]
The difference $\nu_m-\nu_n$ is normal with mean
$(K-1)(\mu_m-\mu_n)$ and variance $2(K-2)\sigma^2$. Thus
\[
\E\!\left[e^{-(\nu_m-\nu_n)^2/(2\sigma^2)}\right]
=\frac{\exp\!\left[-\frac{(K-1)^2(\mu_m-\mu_n)^2}
 {2\sigma^2(2K-3)}\right]}{\sqrt{2K-3}}
\le\frac1{\sqrt{2K-3}}.
\]
There are $N(N-1)$ off-diagonal entries, so
\[
\E[\varepsilon_{W_1,1}\varepsilon_{W_2,2}]
\le\frac{N^2(N-1)\sigma^2}{4\pi\sqrt{2K-3}},
\qquad \Var(A)\le\alpha_{N,K}.
\]
Using unbiasedness and Cauchy--Schwarz with possibly dependent $A$ and $D$ gives
\[
L_K(\hbon;\boldsymbol\mu,\sigma)=\Var(A-D)
\le(\sqrt{\alpha_{N,K}}+\sqrt{\beta_{N,K}})^2.
\]

For $M>N$, exact subset averaging gives
$\hbon=\E_S[H_S\mid R]$ and
$\Theta^{(K-1)}=\E_S[\Theta_S^{(K-1)}]$.
Jensen's inequality bounds the squared error of their difference by
$\E_S[(H_S-\Theta_S^{(K-1)})^2\mid R]$.
Taking expectation over the scores and applying the preceding uniform
$N$-candidate bound proves the same upper bound for every $M\ge N$.

\paragraph{A lower bound from two tied best candidates.}
For $M=2$, the two-candidate UMVU theorem directly gives the stated lower
bound. Suppose $M>2$ and let
\[
f_t(x,y)=\Theta^{(K-1)}(x,y,-t,\ldots,-t),\qquad t>0.
\]
Write $\Theta_2^{(K-1)}(x,y)$ for the corresponding two-candidate policy value.
For every fixed $(x,y)$,
\begin{equation}
\label{eq:general-gaussian-target-limit}
f_t(x,y)+bt\longrightarrow c\Theta_2^{(K-1)}(x,y)+a(x+y).
\end{equation}
Indeed, subsets containing neither designated candidate have value exactly
$-t$. Subsets containing one have limiting value $x$ or $y$; subsets
containing both have the two-candidate limiting value. The error in each
case is bounded by $(t+|x|+|y|)$ times the probability that an inactive
candidate beats an active one. A union bound and Gaussian pairwise tail
bound make this $O_{x,y,K,M,N}(t e^{-q t^2})$ for some $q>0$ as $t\to\infty$.
This proves the target limit, including its potentially diverging constant.

Fix any $U\in\mathcal U_K(\sigma)$. We show
\begin{equation}
\label{eq:general-gaussian-tied-top-liminf}
\liminf_{t\to\infty}
L_K(U;(0,0,-t,\ldots,-t),\sigma)\ge B_{M,N,K}.
\end{equation}
If the left side is infinite there is nothing to prove. Otherwise choose
$t_i\to\infty$ along a sequence realizing a finite limit inferior.
Let $X=(\bar R_1,\bar R_2)$, and Rao--Blackwellize at inactive means $-t_i$:
\[
v_i(X)=\E_{(0,0,-t_i,\ldots,-t_i),\sigma}[U\mid X]-f_{t_i}(0,0).
\]
Gaussian residuals in the first two rows are independent of their row means
and have laws independent of $(x,y)$; the remaining rows are independent of
the first two. Therefore this conditional-expectation function satisfies
\[
\E_{x,y}[v_i(X)]=f_{t_i}(x,y)-f_{t_i}(0,0)
\quad\hbox{for every }(x,y)\in\mathbb R^2.
\]
In particular, $\E_{0,0}v_i=0$ and
$\|v_i\|_{L^2(P_0)}^2\le
L_K(U;(0,0,-t_i,\ldots,-t_i),\sigma)$, where
$P_0=\mathcal N(0,\sigma^2/K)^{\otimes2}$.
The sequence is bounded in this Hilbert space, so a subsequence converges
weakly to some $v\in L^2(P_0)$.

For every fixed $(x,y)$, the likelihood ratio of
$\mathcal N((x,y),\sigma^2 I_2/K)$ to $P_0$ is
\[
\ell_{x,y}(X)=\exp\!\left\{
\frac K{\sigma^2}(xX_1+yX_2)-\frac K{2\sigma^2}(x^2+y^2)
\right\}\in L^2(P_0).
\]
Indeed, $\E_0[\ell_{x,y}^2]=e^{K(x^2+y^2)/\sigma^2}<\infty$.
Weak convergence therefore passes the expectation against this likelihood
ratio to the limit. Since $\Theta_2^{(K-1)}(0,0)=0$, subtracting the limit
at $(0,0)$ in Equation~\eqref{eq:general-gaussian-target-limit} gives
\[
\E_{x,y}v(X)=c\Theta_2^{(K-1)}(x,y)+a(x+y).
\]
The limit is integrable at every $(x,y)$ by Cauchy--Schwarz with
$\ell_{x,y}$. Completeness of the two normal row means implies
\[
v(X)=cT_{K-1,K}(X)+a(X_1+X_2)
\quad P_0\hbox{-almost surely},
\]
where $T_{K-1,K}$ is the two-candidate Gaussian UMVU estimator from
Equation~\eqref{eq:gaussian-umvu}.
At $(x,y)=(0,0)$, write $S=(X_1+X_2)/2$ and $D=X_1-X_2$.
They are independent, $\Var(S)=\sigma^2/(2K)$, and
$T_{K-1,K}=S+g(D)$ for a centered $g(D)$ of variance
$\mathcal R_{K-1,K}-\sigma^2/(2K)$. Consequently
$\|v\|_{L^2(P_0)}^2=B_{M,N,K}$.
Weak lower semicontinuity of the norm proves
Equation~\eqref{eq:general-gaussian-tied-top-liminf}.
Taking the supremum over mean vectors and then the infimum over $U$ proves
the lower bound in Equation~\eqref{eq:general-gaussian-unbiased-bounds}.

\paragraph{Allowing bias.}
Let $P=\E_S[\max_{m\in S}\bar R_m]$ be the exact subset average of the
observed row-mean maximum, and put
$\Theta_\infty^\star=\E_S[\max_{m\in S}\mu_m]$.
The maximum is 1-Lipschitz in Euclidean norm. Jensen's inequality over subsets
therefore gives
\[
\E[(P-\Theta_\infty^\star)^2]\le N\sigma^2/K.
\]
For a fixed subset and a fixed maximal-mean candidate $m_\star$, the
selector regret is at most $\max_m e_m-e_{m_\star}$.
Since $\E[e_{m_\star}]=0$, its expectation is bounded by
$\E[\max_m e_m-e_{m_\star}]=\E[\max_m e_m]$.
The normal maximum bound gives
\[
0\le\Theta_\infty^\star-\Theta^{(K-1)}
\le\sigma\sqrt{\frac{2\log N}{K-1}}.
\]
It follows that
\begin{equation}
\label{eq:general-gaussian-unrestricted-bounds}
\frac{\sigma^2}{MK}
\le\inf_U\sup_{\boldsymbol\mu}L_K(U;\boldsymbol\mu,\sigma)
\le\sup_{\boldsymbol\mu}L_K(P;\boldsymbol\mu,\sigma)
\le\frac{2N\sigma^2}{K}+\frac{4\sigma^2\log N}{K-1}.
\end{equation}
For the lower bound restrict all $M$ means to $s$, where the target is $s$
and the observations are $MK$ independent $\mathcal N(s,\sigma^2)$ variables.
A $\mathcal N(0,L^2)$ prior has posterior variance
$(L^{-2}+MK/\sigma^2)^{-1}$; letting $L\to\infty$ gives the displayed bound.
The infimum here is over all estimators, with bias permitted.

\paragraph{Unknown common variance.}
Both $\hbon$ and $P$ use no value or estimate of $\sigma$.
Their upper bounds hold at every scale, and $\hbon$ is unbiased for all
$(\boldsymbol\mu,\sigma)$. For normalized loss, the lower bounds restrict
the supremum to $\sigma=1$. Thus the same $K^{-1/2}$ unbiased and $K^{-1}$
unrestricted orders hold for the unknown-variance problem, with the
supremum over all mean vectors and all positive variances.
An unnormalized supremum over unbounded $\sigma$ is infinite.

\section{Subset weights, ties, and computation}
\label{app:computation}
For distinct selection keys in ascending order, rank $i$ wins exactly $\binom{i-1}{N-1}$ size-$N$ subsets: include that candidate and choose the other $N-1$ from lower ranks. Dividing by $\binom MN$ gives $w_i$. Consequently $\sum_iw_i a_i=\E_S[\max_{m\in S}a_m]$ for sorted $a_i$. On binary rewards with $c$ successes, this is $1-\binom{M-c}{N}/\binom MN$, the pass@$k$ identity with $k=N$ \citep{chen2021codex}.

For a tied block occupying ranks $\mathcal B$, a uniform random ordering assigns each candidate each rank with probability $1/|\mathcal B|$, giving expected weight $|\mathcal B|^{-1}\sum_{i\in\mathcal B}w_i$. This ordering makes each tied maximum in every subset equally likely to win, so averaging the weights computes the exact subset average with uniform ties, conditional on the observed matrix. Candidates in a tied block need not have equal means. Stable sorting without this averaging would favor the index order.

Rolling selection sums update by subtracting one column and adding the next; evaluation sums equal row totals minus selection sums. The updates cost $O(MK)$ and the $K$ sorts cost $O(MK\log M)$. For stable weight computation, start at $w_M=N/M$ and recurse downward using $w_{i-1}=w_i(i-N)/(i-1)$ for $i=M,\ldots,N+1$; lower weights are zero. This avoids evaluating large binomial coefficients.

\paragraph{What averaging removes.}
Let $Z_S(R)$ be the cyclic estimate on one uniformly sampled size-$N$ subset, with tied scores averaged. Then $\E_S[Z_S\mid R]=\widehat\Theta_{J,K}(R)$. For $L$ independent subset draws with replacement, total variance gives
\[
\Var\!\left(\frac1L\sum_{\ell=1}^L Z_{S_\ell}\right)
=\Var(\widehat\Theta_{J,K})+\frac1L\E[\Var_S(Z_S\mid R)].
\]
This comparison assumes finite variances and an available full matrix; acquiring only a subset costs fewer calls. Similarly, cyclic averaging is the conditional average of a fixed split after a uniform common column rotation. The matrix law is invariant under that rotation, so cyclic averaging has no greater variance than the fixed split. Averaging all $\binom KJ$ common-column selection sets further averages over column permutations and has no greater variance than the cyclic scheme. At $J=K-1$ the two averages coincide. Classical symmetrization gives these variance inequalities for a fixed observation design.

For a fixed $J<K$, averaging every $\binom KJ$ selection split can reduce
variance at increased computation cost. An intermediate option averages the
cyclic estimator over $L$ independent, uniformly sampled common column
permutations, drawn independently of the scores. Let $C$ denote the cyclic
estimate, $F$ the complete common-column split average, and $C_L$ this
randomized average. Under Assumption~\ref{ass:indep} and finite variance,
all three are unbiased for the same $\Theta^{(J)}$, and
\begin{equation}
\label{eq:permutation-interpolation}
\Var(C_L\mid\mathcal F)
=\Var(F\mid\mathcal F)
+\frac{\Var(C\mid\mathcal F)-\Var(F\mid\mathcal F)}{L}.
\end{equation}
Thus $L=2$ removes half the variance gap between cyclic and complete
splitting. The cost is $O(LMK\log M)$ operations on the existing matrix,
with no additional judge calls. At $J=1$ or $J=K-1$, cyclic and complete
splitting already coincide. Appendix~\ref{app:permutation-interpolation}
proves the identity. This comparison uses a fixed observation design.
\paragraph{Common shifts.}
A shared additive shift $b_k$ in column $k$ changes all selection keys in a block by the same amount. Ranks and ties are preserved. Every column belongs to $K-J$ evaluation complements, giving the deterministic identity
\[
\widehat\Theta_{J,K}(R+\mathbf1_Mb^\top)
=\widehat\Theta_{J,K}(R)+K^{-1}\sum_kb_k.
\]
If the centered matrix satisfies Assumption~\ref{ass:indep}, Theorem~\ref{thm:unbiased} applies to its score laws. The estimate for the shifted matrix is the centered estimate plus the mean column shift.

\paragraph{Averaging random column permutations.}
\label{app:permutation-interpolation}
Fix the realized pool $\mathcal F$, $1\le J<K$, and a finite-variance cyclic
estimate $C(R)=\widehat\Theta_{J,K}(R)$. For a selection set $I$ of size $J$,
write $G_I(R)$ for its exact subset-and-tie average evaluated with the
complementary columns, as in Equation~\eqref{eq:fixed-J-main}. Define
\[
F(R)=\binom KJ^{-1}\sum_{I:\,|I|=J}G_I(R).
\]
Let $\Pi$ be a uniformly random permutation of the $K$ columns, independent
of the scores, acting on every candidate's row in the same way. Each cyclic
selection block in $C(\Pi R)$ becomes a uniform size-$J$ selection set in
the original matrix, with its evaluation complement unchanged as a set.
Consequently
\[
\E_\Pi[C(\Pi R)\mid R,\mathcal F]=F(R).
\]
For iid copies $\Pi_1,\ldots,\Pi_L$, independent of the scores, put
$C_L=L^{-1}\sum_{\ell=1}^L C(\Pi_\ell R)$. Given $R,\mathcal F$, these
summands are independent, so total variance gives
\[
\Var(C_L\mid\mathcal F)
=\Var(F\mid\mathcal F)
+\frac1L\E\!\left[
  \Var_\Pi(C(\Pi R)\mid R,\mathcal F)\,\middle|\,\mathcal F\right].
\]
Under Assumption~\ref{ass:indep}, the score-matrix law is invariant under
common column permutations. Thus $C(\Pi R)$ has the same conditional law
as $C(R)$, and the case $L=1$ identifies the second term's numerator as
\mbox{$\Var(C\mid\mathcal F)-\Var(F\mid\mathcal F)$}. This proves
Equation~\eqref{eq:permutation-interpolation}. Unbiasedness follows from
Theorem~\ref{thm:unbiased} and permutation invariance. The identity uses complete common-column splitting, with each permutation
applied identically to every candidate's row.
% Candidate 7 appendix proofs. No frozen candidate is modified.

\section{Selector depth and bounded-score risk}
\label{app:c7-depth-risk}

\paragraph{Gaussian depth regimes.}
\begin{proof}
Equation~\eqref{eq:gaussian-tie-risk} gives the exact identity.
Near zero,
$\arcsin r+r/\sqrt{1-r^2}=2r+O(r^3)$, proving the fixed-$J$ expansion.
Continuity of $C$ gives the second regime, including $r_0=0$.
For $\delta=s_K/K\downarrow0$,
\[
C(1-\delta)=\frac1{\pi\sqrt{2\delta}}+1+O(\sqrt\delta),
\]
which gives the third regime. More precisely,
$\mathcal R_{K-s_K,K}=\sigma^2/(\pi\sqrt{2Ks_K})+
\sigma^2/K+O(\sigma^2\sqrt{s_K}/K^{3/2})$ when $s_K=o(K)$.
\end{proof}

\paragraph{Bounded fixed-depth evaluation.}
\begin{proof}
Condition throughout on the realized pool. The column vectors
$Z_k=(R_{1,k},\ldots,R_{M,k})$ are iid.
Write $d=J+1$. On $d$ columns define the symmetric kernel
\[
h(Z_1,\ldots,Z_d)=\frac1d\sum_{\ell=1}^d\sum_{m=1}^M
 q_m(Z_{-\ell})R_{m,\ell},
\]
where $q_m(Z_{-\ell})$ is the size-$N$ subset and tie averaged winning
probability when the other $J$ columns select. Independent evaluation
gives $\E h=\Theta^{(J)}$, and the simplex weights give
$a\le h\le a+B$, hence $\Var(h)\le B^2/4$.
The order-$d$ U-statistic that averages $h$ over all $d$-column sets equals
$\widehat\Theta^{\rm all}_{J,K}$: every selection-set/evaluation-column pair
has the same weight because
$d\binom Kd=(K-J)\binom KJ$.

For completeness, let $\zeta_r\ge0$ be the variance of the order-$r$
canonical projection in the Hoeffding decomposition of $h$. Orthogonality
gives
\[
\Var(h)=\sum_{r=1}^d\binom dr\zeta_r,\qquad
\Var(\widehat\Theta^{\rm all}_{J,K})
=\sum_{r=1}^d\frac{\binom dr^2}{\binom Kr}\zeta_r
\le\frac dK\Var(h),
\]
since $\binom dr/\binom Kr\le d/K$ for $1\le r\le d\le K$.
This proves the first bound.

For the cyclic estimator, replacing one column changes the $J$ blocks
that use it for selection by at most $B$ each. In each of the other
$K-J$ blocks, the winning weights are unchanged and its contribution to
the evaluation mean changes by at most $B/(K-J)$. The complete cyclic
average therefore changes by at most $(J+1)B/K$.
The independent-coordinate variance bound
$\Var f\le\sum_k\E[\Var(f\mid Z_{-k})]$ and the conditional range
bound $\Var(f\mid Z_{-k})\le ((J+1)B/K)^2/4$ give the second displayed
bound, combined with $\Var(\widehat\Theta_{J,K})\le B^2/4$.
Unbiasedness turns both variance bounds into MSE bounds.

For the minimax lower bound, let $U$ be integrable and unbiased for
$\Theta^{(J)}$ under every independent, stable collection of laws supported
on $[a,a+B]$. Restrict every candidate law to
$a+B\operatorname{Bernoulli}(p)$, $0<p<1$.
The target is then $a+Bp$ for every $J,N$.
Let $T$ count successes among all $MK$ observations. On this finite
sample space, differentiating unbiasedness gives
$\operatorname{Cov}_p(U,T)=Bp(1-p)$.
Since $\Var_p(T)=MKp(1-p)$, Cauchy--Schwarz at $p=1/2$ yields
$\Var_{1/2}(U)\ge B^2/(4MK)$.
Infinite-variance estimators already satisfy the bound; independent
auxiliary randomization can first be averaged out. Taking the supremum
over the bounded model and infimum over its globally unbiased class
completes the proof.
\end{proof}

\section{Numerical protocols and reproducibility}
\label{app:numerical}
The supplement contains the score-generating scripts, recorded seeds or primitive arrays, per-replication outputs, exact-target checks, and figure data, along with the protocols and their chronology. The finite-$K$ Gaussian grid was fixed after earlier numerical results were available and before generating its new draws. The one-score and dependence comparisons reuse existing matrices. All reported studies use synthetic scores.

\paragraph{Exact Gaussian risk.}
The four curves in Figure~\ref{fig:gaussian-risk} use Equations~\ref{eq:gaussian-tie-risk} and~\ref{eq:gaussian-holdout-risk}, with $1/K$ for the biased row-mean maximum. The \texttt{gaussian-diagonal-risk} package includes exact figure values, plotting code, and deterministic checks using Gaussian quadrature and direct matrix evaluations. For numerical evaluation of Equation~\ref{eq:gaussian-umvu}, set $q=\sqrt{J/[2(K-J)]}$ and $z=D/\tau$, giving correction $\tau q e^{-q^2z^2}/\sqrt\pi$. Computing the positive integer $K-J$ directly avoids cancellation in $1-J/K$.

\paragraph{Discrete-score check.}
Five fixed Bernoulli pools use
\[
(M,N,K)\in\{(2,2,2),(2,2,4),(4,2,4),(4,4,4),(8,4,8)\}.
\]
Means are drawn once per pool from $\mathrm{Unif}(0.2,0.8)^M$ with seed \texttt{2026092203}; each pool has $200{,}000$ independent matrices from score seed \texttt{2026092204}. Finite binomial sums with uniform ties compute both $\Theta^{(K-1)}$ and $\Theta^{(K)}$. The mean-scored selector evaluates each matrix's realized selections using the selected candidates' known means; its expectation is the integrated policy target. In each cell, the paired mean difference (Holdout minus this selector) is within $1.70$ times its Monte Carlo standard error of zero. The \texttt{bernoulli-reproduction} directory contains every cell and output.

\paragraph{Gaussian reuse controls.}
At $M=16,N=4$, four cells cross $K\in\{4,8\}$ with score SD $\sigma\in\{0.2,0.8\}$. Each $K$ has $40{,}000$ iid $\mathrm{Beta}(2,2)^{16}$ mean pools and independent standard-normal score shocks, shared across its noise-scale cells. Holdout and reuse use identical leave-one-out selections, subset averaging, and ties. Each difference from the mean-scored selector is paired by pool and matrix; its Monte Carlo SE is the sample SD of those differences divided by $\sqrt{40{,}000}$. Holdout differences are within $1.81$ such SEs of zero; reuse differences range from $+0.0243$ to $+0.4277$. The \texttt{compact-synthetic} package retains the full protocols, primitives, and outputs, including separate dependence studies not used to claim real-call independence.

\paragraph{Fixed one-score comparison.}
Table~\ref{tab:one-score} reanalyzes the saved Gaussian matrices in the \texttt{adaptive-budget} study: $M=N=8$, $20{,}000$ independent pools per $K\in\{4,8\}$, iid $\mathrm{Unif}(0.2,0.8)$ means, and score SD $0.6$. Cyclic split evaluation selects with each column and evaluates with the other $K-1$ scores, then averages over columns. The fixed split uses the first column's winner; reuse averages the winning selection scores. All use the same one-score selection rule and are compared with its value; reuse is generally biased. Gaussian quadrature integrates each pool's $\Theta^{(1)}$, with convergence checks in the \texttt{one-score} package. Its paired-bootstrap intervals, saved matrices, and numerical verification distinguish pool-specific RMSE from the total population variance used in cost comparisons. 

\paragraph{Prespecified finite-$K$ risks for larger pools.}
The prespecified grid uses $(M,N)\in\{(4,2),(8,4),(8,8)\}$, $K\in\{4,16,64\}$, common Gaussian score variance one, and three fixed mean vectors: all zero; $(0,0,-1,-2,\ldots)$; and $(0,-1,-2,\ldots)$.

These $27$ cells were fixed before generating draws for the supplement's \texttt{finite-gaussian-risk} package. One NumPy PCG64 stream with seed words $[2026092307,1]$ generates a saved $50{,}000\times8\times64$ array of independent standard normals. Smaller $M,K$ use prefixes of this array, and all fixed mean configurations use the same normals. The cells are coupled through the shared normals. Each cell contains $50{,}000$ independent matrices. Exact subset weights and tie averaging compute Holdout and the row-mean maximum comparator.

For fresh selection means of variance $1/(K-1)$, candidate $m$'s win probability is
\[
p_m=\binom MN^{-1}\int f_m(x)\,e_{N-1}\!\left(\{F_\ell(x):\ell\ne m\}\right)dx,
\qquad \Theta^{(K-1)}=\sum_m\mu_mp_m,
\]
where $e_{N-1}$ is the elementary symmetric polynomial and $f_m,F_m$ are the corresponding normal density and CDF. Gauss--Hermite orders $64,128,256$ check convergence; the last two targets differ by less than $10^{-12}$ in every cell. A separate implementation explicitly enumerates subsets and applies adaptive one-dimensional integration, agreeing within $2\times10^{-11}$. Winning probabilities are used without renormalization.

For error $e=\widehat\Theta-\Theta$, reported MSE is the mean of $e^2$ and its MCSE is $\operatorname{sd}(e^2)/\sqrt{50{,}000}$, using sample SD. The paired difference uses $\operatorname{sd}(e_H^2-e_P^2)/\sqrt{50{,}000}$; bias uses $\operatorname{sd}(e)/\sqrt{50{,}000}$. These quantify Monte Carlo error conditional on each fixed mean vector. The full grid, primitive normals, estimates, protocol freeze, and independent recomputation are supplied.

\end{document}